\documentclass[11pt,a4paper]{article}
\usepackage[hyperref]{acl2017}
\usepackage{times}
\usepackage{latexsym}

\usepackage{url}

\usepackage[utf8]{inputenc}

\usepackage{graphicx}
\usepackage{amsmath}
\usepackage{amsfonts}
\usepackage{times}
\usepackage{url}
\usepackage{latexsym}
\usepackage{graphicx}
\usepackage{verbatim}
\usepackage{booktabs}
\usepackage{subfig}
\usepackage{amsthm}
\usepackage{qtree}
\usepackage{graphicx}
\usepackage{caption}
\usepackage{multirow}
\usepackage{mathrsfs}
\usepackage{tablefootnote}
\usepackage{arydshln}
\usepackage{lscape}
\usepackage{afterpage}
\usepackage{todonotes}
\usepackage{arydshln}
\usepackage{color}
\usepackage[linewidth=1pt]{mdframed}

\newtheorem{prop}{Proposition}

\DeclareMathOperator*{\argmax}{arg\,max}

\aclfinalcopy 
\def\aclpaperid{***} 

\title{Optimizing Differentiable Relaxations of Coreference Evaluation Metrics}

\author{
 Phong Le$^1$ \hspace{1cm}  Ivan Titov$^{1,2}$    \\
 $^1$ILLC,  University of Amsterdam  \\
   $^2$ILCC, School of Informatics, University of Edinburgh \\
    {\tt p.le@uva.nl} \hspace{1cm} {\tt ititov@inf.ed.ac.uk}
 }

\begin{document}

\maketitle

\begin{abstract}
Coreference evaluation metrics are hard to optimize directly as they are non-differentiable
functions, not easily decomposable into elementary decisions. Consequently,
most approaches optimize objectives only indirectly related to the end goal, resulting
in suboptimal performance.
Instead, we propose a differentiable relaxation that lends itself to gradient-based optimisation, 
thus bypassing the need for reinforcement learning or heuristic modification of cross-entropy.
We show that by modifying the training objective of a competitive neural coreference
system, we obtain a substantial gain in performance. 
This suggests
that our approach can be regarded as a viable alternative to using reinforcement learning or more computationally expensive imitation learning. 
    
\end{abstract}

\section{Introduction}


Coreference resolution is the task of identifying all mentions 
which refer to the same entity in a document. 
It has been shown beneficial in many natural language processing (NLP) applications, including question answering~\cite{nips15_hermann}
and information extraction~\cite{W97-0319}, and often regarded as a prerequisite to any text understanding task.

Coreference resolution can be regarded as a clustering problem:
each cluster corresponds to a single entity and consists of all its  mentions in a given text.
Consequently, it is natural to evaluate predicted clusters by comparing them with the ones annotated by human experts,
and this 
is exactly what the standard metrics (e.g., MUC, B$^3$, CEAF) do. In contrast, most state-of-the-art systems
are optimized to make individual co-reference decisions, and such losses are only indirectly related to the metrics.

One way to deal with this challenge is to optimize directly the non-differentiable 
metrics using reinforcement learning (RL), for example,  relying on the REINFORCE policy gradient algorithm~\cite{williams1992simple}.
However, this approach has not been very successful, which, as suggested by \newcite{clark-manning:2016:EMNLP2016}, is possibly due to the discrepancy between sampling decisions at training time and choosing the highest ranking ones at test time.  A more successful alternative is using a `roll-out' stage to associate cost with possible decisions, as in \newcite{clark-manning:2016:EMNLP2016}, but it is computationally expensive. Imitation learning \cite{D14-1225,P15-1136}, though also exploiting metrics, requires access to an expert policy, with exact policies not directly computable for the metrics of interest.



In this work, we aim at combining the best of both worlds by proposing 
a simple method that can turn popular coreference evaluation metrics into
differentiable functions of model parameters. As we show, this function
can be computed recursively using scores of individual local decisions, resulting
in a simple and efficient estimation procedure. The key idea is to replace
non-differentiable indicator functions (e.g. the member function $\mathbb{I}(m \in S)$) with
the corresponding posterior probabilities ($p(m \in S)$) computed by the model. Consequently,
non-differentiable functions used within the metrics
(e.g. the set size function $|S| = \sum_m \mathbb{I}(m \in S)$)
become differentiable   ($|S|_c = \sum_m p(m \in S)$). Though we assume that the scores of the underlying
statistical model can be used to define a probability model, we show that this is not a serious limitation. Specifically,
as a baseline we use a probabilistic version of the neural mention-ranking model of~\newcite{P15-1137}, which
on its own outperforms the original one and achieves similar performance to its global version~\cite{N16-1114}. 
Importantly when we use the introduced differentiable relaxations in training, we observe a substantial gain
in performance over our probabilistic baseline. Interestingly, the absolute improvement (+0.52)
is higher than the one reported in \newcite{clark-manning:2016:EMNLP2016} using RL (+0.05) and 
the one using reward rescaling\footnote{Reward rescaling is a technique that computes error values for a heuristic loss function
based on the reward difference between the best decision according to the current model 
and the decision leading to the highest metric score.}
 (+0.37). 
This suggests that 
our method provides a viable alternative to using RL and reward rescaling.

The outline of our paper is as follows: we introduce our neural resolver baseline and
the B$^3$ and LEA metrics in Section~\ref{sec background}. Our method to turn 
a mention ranking resolver into an entity-centric resolver is presented in 
Section~\ref{sec ranking to entity}, and the proposed differentiable relaxations in 
Section~\ref{sec differentiable metrics}. 
Section~\ref{sec experiments} shows our experimental results.

\section{Background}
\label{sec background}
\subsection{Neural mention ranking}
\label{subsec neural mr}

In this section we introduce neural mention ranking, the framework which underpins current  state-of-the-art models~\cite{clark-manning:2016:EMNLP2016}.
%
Specifically, we consider a probabilistic version of the method proposed by \newcite{P15-1137}. 
In experiments we will use it as our baseline.

Let $(m_1, m_2,.., m_n)$ be the list of mentions in a document.  
For each mention $m_i$, let $a_i \in \{1, ..., i\}$ be the index of 
the mention that $m_i$ is coreferent with (if $a_i=i$, $m_i$ is the 
first mention of some entity appearing in the document). As standard
in coreference resolution literature, we will refer to $m_{a_i}$ as an antecedent
of $m_i$.\footnote{This slightly deviates from the definition of antecedents in linguistics~\cite{Crystal:97}.}  
Then, in mention ranking the goal is to score 
antecedents of a mention higher than any other 
mentions, i.e., if $\mathsf{s}$ is the scoring function, we require $\mathsf{s}(a_i=j) > \mathsf{s}(a_i=k)$ for all $j,k$ 
such that $m_i$ and $m_j$ are coreferent but $m_i$ and $m_k$ are not.

Let $\phi_a(m_i) \in \mathbb{R}^{d_a}$ and $\phi_p(m_i, m_j) \in \mathbb{R}^{d_p}$ 
be respectively features of $m_i$ and features of pair $(m_i,m_j)$.
The scoring function is defined by:
\begin{equation*}
    \mathsf{s}(a_i=j) = 
    \begin{cases}
        \mathbf{u}^T \begin{bmatrix}
        \mathbf{h}_a(m_i) \\ \mathbf{h}_p(m_i, m_j) 
        \end{bmatrix} + u_0 & \text{if } j < i \\
        \mathbf{v}^T \mathbf{h}_a(m_i) + v_0 & \text{if } j = i
    \end{cases}
\end{equation*}
where 
\begin{align*}
    \mathbf{h}_a(m_i) &= \tanh(\mathbf{W}_a \phi_a(m_i) + \mathbf{b}_a) \\
    \mathbf{h}_p(m_i, m_j) &= \tanh(\mathbf{W}_p \phi_p(m_i,m_j) + \mathbf{b}_p)
\end{align*}
and $\mathbf{u}, \mathbf{v}, \mathbf{W}_{a}, \mathbf{W}_{p}, \mathbf{b}_{a}, \mathbf{b}_{p}$ 
are real vectors and matrices with proper dimensions, $u_0, v_0$ are real scalars.

Unlike 
\newcite{P15-1137}, where  
the max-margin loss is used, we define a probabilistic model.
The probability\footnote{For the sake of readability, we do not explicitly mark
in our notation that all the probabilities are conditioned on the document (e.g., the mentions) and dependent on model parameters.}
that $m_i$ and $m_j$ are coreferent is given by
\begin{equation}
    p(a_i=j) = \frac{\exp\{\mathsf{s}(a_i=j)\}}{\sum_{j'=1}^{i}\exp\{\mathsf{s}(a_i=j')\}}
    \label{equ p_link}
\end{equation}
Following \newcite{D13-1203} we use the following softmax-margin \cite{gimpel-smith:2010:NAACLHLT} loss function:
\begin{equation*}
    L(\Theta) = - \sum_{i=1}^n \log \big(\sum_{j \in C(m_i)} p'(a_i=j) \big) + \lambda ||\Theta||_1,
\end{equation*}
where  $\Theta$ are model parameters, $C(m_i)$ is the set of the indices of 
correct antecedents of $m_i$, and $p'(a_i=j) \propto p(a_i=j)e^{\Delta(j,C(m_i))}$.
$\Delta$ is a cost function used to manipulate the
contribution of different error types to the loss function: 
\begin{equation*}
    \Delta(j,C(m_i)) = 
    \begin{cases}
        \alpha_1 & \text{if } j \ne i \wedge i \in C(m_i)\\
        \alpha_2 & \text{if } j = i \wedge i \notin C(m_i) \\
        \alpha_3 & \text{if } j \ne i \wedge j \notin C(m_i) \\
        0 & \text{otherwise}
    \end{cases}
\end{equation*}
The error types are ``false anaphor'', ``false new'', ``wrong link'', and ``no mistake'', respectively.
In our experiments, we borrow their values from \newcite{D13-1203}:
$(\alpha_1, \alpha_2, \alpha_3)=(0.1, 3, 1)$.
In the subsequent discussion, we refer to the loss as 
{\it mention-ranking heuristic cross entropy}.

\subsection{Evaluation Metrics}
\label{sub metric}
We use five most popular metrics\footnote{All are implemented in \newcite{pradhan-EtAl:2014:P14-2}, \url{https://github.com/conll/reference-coreference-scorers}.}, 
\begin{itemize}
\item MUC \cite{M95-1005}, 
\item B$^3$ \cite{bagga1998algorithms},
\item CEAF$_{m}$, CEAF$_{e}$ \cite{H05-1004}, 
\item BLANC \cite{P14-2005}, 
\item LEA \cite{moosavi-strube:2016:P16-1}. 
\end{itemize}
for evaluation. However, because MUC is the least discriminative metric
\cite{moosavi-strube:2016:P16-1},
whereas CEAF is slow to compute, out of the five most popular 
metrics we incorporate into our loss only B$^3$. In addition, we  
integrate LEA, as it  has been shown to provide a good balance between discriminativity 
and interpretability.

Let $G=\{G_1, G_2, ..., G_N\}$ and $S=\{S_1, S_2, ..., S_M\}$ be 
the gold-standard entity set and an entity set given by a resolver. 
Recall that an entity is a set of mentions.
The recall and precision of the B$^3$ metric is computed by:
\begin{align*}
    R_{B^3} &= \frac{\sum_{v=1}^N \sum_{u=1}^M \frac{|G_v \cap S_u|^2}{|G_v|} }{\sum_{v=1}^N |G_v|} \\
    P_{B^3} &= \frac{\sum_{u=1}^M \sum_{v=1}^N \frac{|G_v \cap S_u|^2}{|S_u|} }{\sum_{u=1}^M |S_u|}
\end{align*}
The LEA metric is computed as:
\begin{align*}
    R_{LEA} &= \frac{\sum_{v=1}^N \big(|G_v| \times \sum_{u=1}^M \frac{link(G_v \cap S_u)}{link(G_v)}\big)}{\sum_{v=1}^N |G_v|} \\
    P_{LEA} &= \frac{\sum_{u=1}^M \big(|S_u| \times \sum_{v=1}^N \frac{link(G_v \cap S_u)}{link(S_u)}\big)}{\sum_{u=1}^M |S_u|} \\
\end{align*}
where $link(E) = |E| \times (|E| -1 ) / 2$ is the number of coreference links 
in entity $E$. $F_{\beta}$, for both metrics, is defined by:
\begin{equation*}
    F_{\beta} = (1 + \beta^2) \frac{P \times R}{\beta^2 P + R}
\end{equation*}
$\beta = 1$ is used in the standard evaluation.

\section{From mention ranking to entity centricity}
\label{sec ranking to entity}

Mention-ranking resolvers do not explicitly provide information about entities/clusters 
which is required by B$^3$ and LEA. We therefore propose a simple solution 
that can turn a mention-ranking resolver into an entity-centric one. 

\begin{figure}
    \centering
    \includegraphics[width=0.47\textwidth]{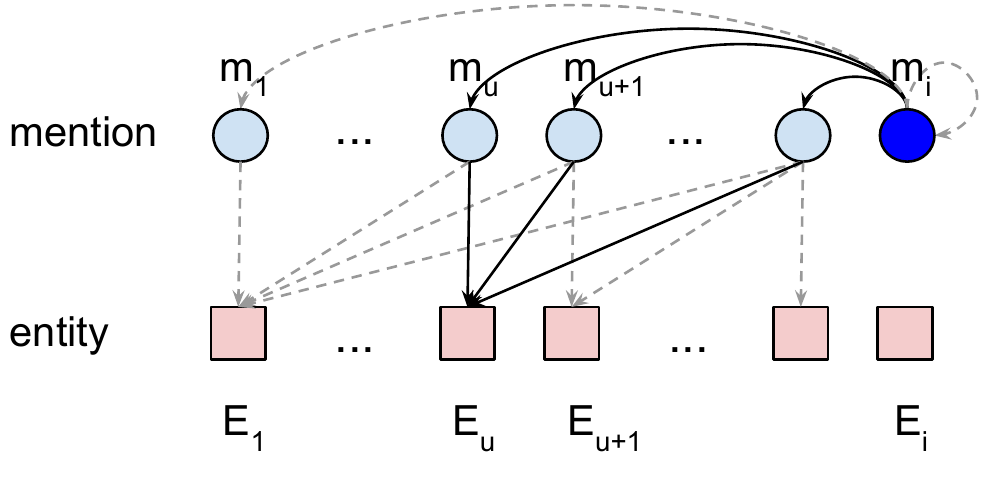}
    \caption{For each mention $m_u$ there is a potential entity $E_u$ so that $m_u$ 
    is the first mention in the chain. Computing $p(m_i \in E_{u}), u < i$  takes into the account
    all directed paths from $m_i$ to $E_u$ (black arrows). 
    Noting that there is no directed path from any $m_{k}, k < u$ to $E_u$
    because $p(m_{k} \in E_u) = 0$. (See text for more details.)}
    \label{fig mention-entity}
\end{figure}

First note that in a document containing $n$ mentions, there are $n$ potential 
entities $E_1, E_2,..., E_n$ where $E_i$ has $m_i$ as the first mention. 
Let $p(m_i \in E_u)$ be the probability that mention $m_i$ corresponds to entity $E_u$. 
We now show that it can be computed recursively based on 
$p(a_i = j)$ 
as follows:
\begin{align*}
    &p(m_i \in E_u) = \\
    &\begin{cases}
    \sum_{j=u}^{i - 1} p(a_i = j) \times p(m_j \in E_u)       & \text{if } u < i \\
    p(a_i = i)  & \text{if } u = i \\
    0 & \text{if } u > i \\
  \end{cases}
\end{align*}
In other words, if $u < i$, we consider all possible $m_j$ with which $m_i$ can be 
coreferent, and which can correspond to entity $E_u$. If $u = i$,  
the link to be considered is the $m_i$'s self-link. And, if $u > i$, the probability is zero,
as it is impossible for $m_i$ to be assigned to an entity introduced only later. 
See Figure~\ref{fig mention-entity} for extra information.

We now turn to two crucial questions about this formula: 
\begin{itemize}
\it
    \item Is $p(m_i \in \bullet)$ a valid probability distribution?
    \item Is it possible for a mention $m_u$ to be mostly anaphoric (i.e. $p(m_u \in E_u)$ is low) but for the corresponding cluster $E_u$ to be highly probable (i.e. $p(m_i \in E_u)$ is high for some $i$)?
\end{itemize}
The first question is answered in Proposition~\ref{prop validity}. 
The second question is important because, intuitively, 
when a mention $m_u$ is anaphoric, the potential entity $E_u$ 
does not exist. We will show that the answer is ``No'' by 
proving in Proposition~\ref{prop consistency} that 
the probability that $m_u$ is anaphoric is always higher than 
any probability that $m_i$, $i > u$ refers to $E_u$.

\begin{prop} $p(m_i \in \bullet)$ is a valid probability distribution, i.e., 
$\sum_{u=1}^n p(m_i \in E_u) = 1$, for all $i=1,...,n$.
\label{prop validity}
\end{prop}

\begin{proof}
We prove this proposition by induction. 

Basis: it is obvious that $\sum_{u=1}^n p(m_1 \in E_u) = p(a_1 = 1) = 1$.

Assume that $\sum_{u=1}^n p(m_j \in E_u) = 1$ for all $j < i$. Then,
\begin{align*}
&\sum_{u=1}^{i-1} p(m_i \in E_u) \\
&= \sum_{u=1}^{i-1} \sum_{j=u}^{i-1} p(a_i = j) \times p(m_j \in E_u)
\end{align*}
Because $p(m_j \in E_u) = 0$ for all $j < u$, this expression is equal to
\begin{align*}
&\sum_{u=1}^{i-1} \sum_{j=1}^{i-1} p(a_i = j) \times p(m_j \in E_u) \\
&=\sum_{j=1}^{i-1} p(a_i = j) \times \sum_{u=1}^{i-1} p(m_j \in E_u) \\
&= \sum_{j=1}^{i-1} p(a_i = j)
\end{align*}
    
Therefore, 
\begin{equation*}
    \sum_{u=1}^n p(m_i \in E_u) = \sum_{j=1}^{i-1} p(a_i = j) + p(a_i = i) = 1
\end{equation*} 
(according to Equation~\ref{equ p_link}).

\end{proof}

\begin{prop}
$p(m_i \in E_u) \leq p(m_u \in E_u)$ for all $i > u$. 
\label{prop consistency}
\end{prop}

\begin{proof}
We prove this proposition by induction. 

Basis: for $i = u + 1$,
\begin{align*}
    p(m_{u+1} \in E_u) &= p(a_{u+1} = u) \times p(m_u \in E_u) \\
    & \le p(m_u \in E_u)
\end{align*}

Assume that $p(m_j \in E_u) \leq p(m_u \in E_u)$ for all $j \ge u$ and $j < i$. Then
\begin{align*}
p(m_i \in E_u) &= \sum_{j=u}^{i-1} p(a_i = j) \times p(m_j \in E_u) \\
&\le \sum_{j=u}^{i-1} p(a_i=j) \times p(m_u \in E_u) \\
&\le p(m_u \in E_u) \times \sum_{j=1}^i p(a_i = j) \\
&= p(m_u \in E_u)
\end{align*}
\end{proof}

\subsection{Entity-centric heuristic cross entropy loss}
\label{sub vanilla loss}
Having $p(m_i \in E_u)$ computed, we can consider coreference resolution as a multiclass 
prediction problem. An entity-centric heuristic cross entropy loss is thus given below:
\begin{equation*}
    L_{ec}(\Theta) = - \sum_{i=1}^n \log p'(m_i \in E_{e(m_i)}) + \lambda ||\Theta||_1
\end{equation*}
where $E_{e(m_i)}$ is the correct entity that $m_i$ belongs to, 
$p'(m_i \in E_u) \propto p(m_i \in E_u) e^{\Gamma(u,e(m_i))}$. 
Similar to $\Delta$ in the mention-ranking heuristic loss in Section~\ref{subsec neural mr}, 
$\Gamma$ is a cost function used to manipulate the contribution of
the four different error types (``false anaphor'', ``false new'', 
``wrong link'', and ``no mistake''): 
\begin{align*}
    &\Gamma(u,e(m_i)) = \\
    &\begin{cases}
        \gamma_1 & \text{if } u \ne i \wedge e(m_i) = i \\
        \gamma_2 & \text{if } u = i \wedge e(m_i) \ne i \\
        \gamma_3 & \text{if } u \ne e(m_i) \wedge u \ne i \wedge e(m_i) \ne i \\
        0 & \text{otherwise}
    \end{cases}
\end{align*}

\section{From non-differentiable metrics to differentiable losses}
\label{sec differentiable metrics}

\begin{figure}
    \centering
    \includegraphics[width=0.5\textwidth]{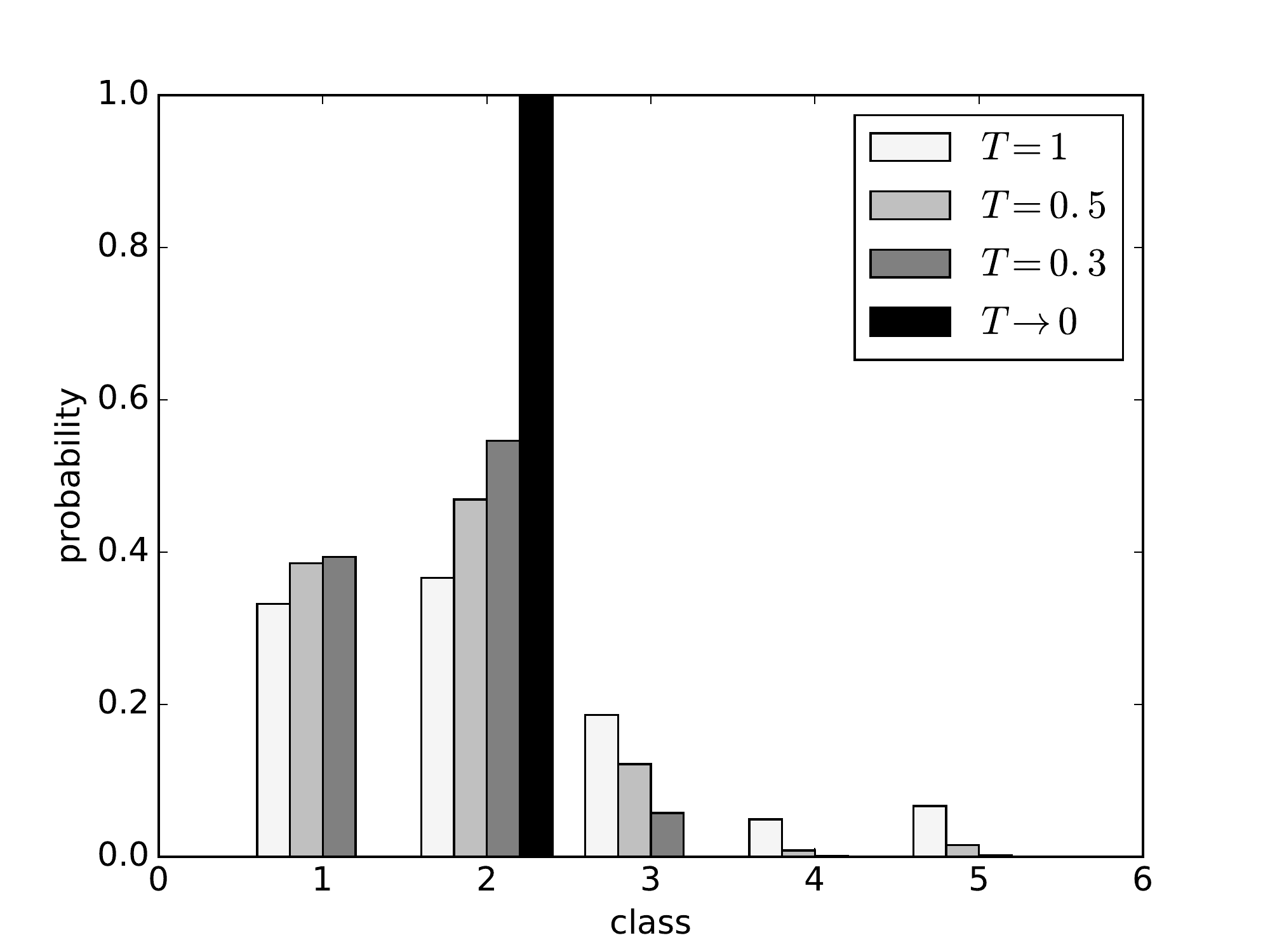}
    \caption{Softmax $\frac{\exp\{\pi_i / T\}}{\sum_j \exp\{\pi_j / T\}}$ 
    with different values of $T$. The softmax becomes more peaky when the value of $T$ 
    gets smaller. As $T \rightarrow 0$ the softmax converges to the indicator function
    that chooses $\argmax_i \pi_i$.}
    \label{fig softmax-T}
\end{figure}

There are two functions used in computing B$^3$ and LEA: 
the set size function $|.|$ and the link function $link(.)$. 
Because both of them are non-differentiable, the two metrics are non-differentiable.
We thus need to make these two functions differentiable. 

There are two remarks. Firstly, both functions can be computed using
the indicator function $\mathbb{I}(m_i \in S_u)$:
\begin{align*}
    |S_u| &= \sum_{i=1}^n \mathbb{I}(m_i \in S_u) \\
    link(S_u) &= \sum_{j < i} \mathbb{I}(m_i \in S_u) \times \mathbb{I}(m_j \in S_u)
\end{align*}
Secondly, given $\pi_{i,u} = \log p(m_i \in S_u)$, 
the indicator function $\mathbb{I}(m_i \in S_{u^*})$, $u^*=\argmax_u p(m_i \in S_u)$ 
is the converging point 
of the following softmax as $T \rightarrow 0$ (see Figure~\ref{fig softmax-T}):
\begin{equation*}
    p(m_i \in S_u; T) = \frac{\exp\{\pi_{i,u} / T\}}{\sum_v \exp\{\pi_{i,v} / T\}}
\end{equation*}
where $T$ is called \textit{temperature} \cite{kirkpatrick1983optimization}.

Therefore, we propose to represent each $S_u$ as a \textit{soft}-cluster:
\begin{equation*}
    S_u = \{ p(m_1 \in E_u; T), ..., p(m_n \in E_u; T) \}
\end{equation*}
where, as defined in Section~\ref{sec ranking to entity},
$E_u$ is the potential entity that has $m_u$ as the first mention. 
Replacing the indicator function $\mathbb{I}(m_i \in S_u)$ by 
the probability distribution $p(m_i \in E_u; T)$, we then have 
a \textit{differentiable} version for the set size function and the link function:
\begin{align*}
    |S_u|_d &= \sum_{i=1}^n p(m_i \in E_u; T) \\
    link_d(S_u) &= \sum_{j < i} p(m_i \in E_u; T) \times p(m_j \in E_u; T)
\end{align*}
$|G_v \cap S_u|_d$ and $link_d(G_v \cap S_u)$ are computed similarly with the  
constraint that only mentions in $G_v$ are taken into account. Plugging these 
functions into precision and recall of B$^3$ and LEA in Section~\ref{sub metric}, 
we obtain differentiable $\hat{F}_{\beta, B^3}$ and  $\hat{F}_{\beta,LEA}$, 
which are then used in two loss functions:
\begin{align*}
    L_{\beta, B^3}(\Theta; T) &= -\hat{F}_{\beta, B^3}(\Theta; T) + \lambda ||\Theta||_1 \\
    L_{\beta, LEA}(\Theta; T) &= -\hat{F}_{\beta, LEA}(\Theta; T) + \lambda ||\Theta||_1
\end{align*}
where $\lambda$ is the hyper-parameter of the $L_1$ regularization terms. 

It is worth noting that, as $T \rightarrow 0$, $\hat{F}_{\beta, B^3} \rightarrow F_{\beta, B^3}$
and $\hat{F}_{\beta, LEA} \rightarrow F_{\beta, LEA}$.\footnote{We can easily prove this 
using the algebraic limit theorem.} 
Therefore, when training 
a model with the proposed losses, we can start at a high temperature (e.g., $T=1$) and 
anneal to a small but non-zero temperature. However, in our experiments 
we fix $T=1$. Annealing is left for future work.

\section{Experiments}
\label{sec experiments}

We now demonstrate how to use the proposed differentiable B$^3$ and 
LEA to train a coreference resolver. The source code and 
trained models are available at \url{https://github.com/lephong/diffmetric_coref}.

\subsection*{Setup}
We run experiments on the 
English portion of CoNLL 2012 data \cite{W12-4501} which consists of 3,492 documents 
in various domains and formats. The split provided in the CoNLL 2012 shared task is used. 
In all our resolvers, we use not the original features of  \newcite{P15-1137} 
but their slight modification described in \newcite{N16-1114} (section 6.1).\footnote{\url{https://github.com/swiseman/nn_coref/}}

\subsection*{Resolvers}
We build following baseline and three resolvers:

\begin{itemize}
\item baseline: the resolver presented 
in Section~\ref{subsec neural mr}. 
We use the identical configuration as in \newcite{N16-1114}: 
$\mathbf{W}_a \in \mathbb{R}^{200 \times d_a}$, 
$\mathbf{W}_p \in \mathbb{R}^{700 \times d_p}$,  
$\lambda=10^{-6}$ 
(where $d_a, d_p$ are respectively the numbers of mention features 
and pair-wise features).
We also employ their pretraining methodology.

\item $L_{ec}$: the resolver using the entity-centric cross entropy loss 
introduced in Section~\ref{sub vanilla loss}. We set
$(\gamma_1, \gamma_2, \gamma_3) = (\alpha_1, \alpha_2, \alpha_3) = (0.1, 3, 1)$. 

\item $L_{\beta,B^3}$ and $L_{\beta,LEA}$: the resolvers using the losses 
proposed in Section~\ref{sec differentiable metrics}. $\beta$ is tuned on 
the development set by 
trying each value in $\{\sqrt{0.8}, 1, \sqrt{1.2}, \sqrt{1.4}, \sqrt{1.6}, \sqrt{1.8}, 
1.5, 2\}$.
\end{itemize}
To train these resolvers we use AdaGrad \cite{duchi2011adaptive} 
to minimize their loss functions with the learning rate tuned on the development set
and with one-document mini-batches. 
Note that we use the baseline as the initialization point 
to train the other three resolvers.


\begin{table*}[t!]
    \centering
    \begin{tabular}{l||c|c|c|c|c|c||c}
        & MUC & B$^3$ & CEAF$_m$ & CEAF$_e$ & BLANC & LEA & CoNLL \\
        \hline \hline
        \newcite{P15-1137} & 72.60 & 60.52 & - & 57.05 & - & - & 63.39 \\
        \newcite{N16-1114} & 73.42 & 61.50 & - & 57.70 & - & - & 64.21 \\
        \bottomrule
        Our proposals & & & & & & & \\
         baseline (heuristic loss) & 73.22 & 61.44 & 65.12 & 57.74 & 62.16 & 57.52 & 64.13 \\
        $L_{ec}$ & 73.2 & 61.75 & 65.77 & 57.8 & \textbf{63.3} & 57.89 & 64.25 \\
        \hdashline
        $L_{\beta=1,B^3}$ & 73.37              & 61.94 & 65.79 & 58.22 & 63.19 & 58.06 & 64.51 \\
        $L_{\beta=\sqrt{1.4},B^3}$ & 73.48 & 61.99 & 65.9 & 58.36 & 63.1 & 58.13 & 64.61 \\ 
        \hdashline
        $L_{\beta=1,LEA}$ & 73.3 & 61.88 & 65.69 & 57.99 & 63.27 & 58.03 & 64.39  \\
        $L_{\beta=\sqrt{1.8},LEA}$ & \textbf{73.53} &  \textbf{62.04} & \textbf{65.95} & \textbf{58.41} & 63.09 & \textbf{58.18} & \textbf{64.66} \\
        \hline \hline
        \newcite{clark-manning:2016:EMNLP2016} & & & & & & & \\
        baseline (heuristic loss) & 74.65 & 63.03 & - & 58.40 & - & - & 65.36 \\
        REINFORCE & 74.48 & 63.09 & - & 58.67 & - & - & 65.41 \\
        Reward Rescaling & 74.56 & 63.40 & - & 59.23 & - & - & 65.73 \\

    \end{tabular}
    \caption{Results (F$_1$) on CoNLL 2012 test set. CoNLL is the 
    average of MUC, B$^3$, and CEAF$_e$.}
    \label{table test}
\end{table*}

\subsection{Results}
We firstly compare our resolvers against \newcite{P15-1137} and \newcite{N16-1114}. 
Results are shown in the first half of Table \ref{table test}. 
Our baseline surpasses \newcite{P15-1137}. 
It is likely due to using features from \newcite{N16-1114}.
Using the entity-centric heuristic cross entropy loss and the relaxations are clearly beneficial: 
$L_{ec}$ is slightly better than our baseline and on par with 
the global model of \newcite{N16-1114}. $L_{\beta=1,B^3}, L_{\beta=1,LEA}$ 
outperform the baseline, the global model of \newcite{N16-1114}, and $L_{ec}$. 
However, the best values of $\beta$ are $\sqrt{1.4}$,
$\sqrt{1.8}$ respectively for $L_{\beta,B^3}$, and $L_{\beta,LEA}$. 
Among these resolvers, $L_{\beta=\sqrt{1.8},LEA}$ achieves the highest F$_1$ scores across 
all the metrics except BLANC.

When comparing to~\newcite{clark-manning:2016:EMNLP2016} 
(the second half of Table \ref{table test}), we can see that
the absolute improvement over the baselines (i.e. `heuristic loss' for them and 
the heuristic cross entropy loss for us)  
is higher than that of reward rescaling but with much shorter training time:
$+0.37$ (7 days\footnote{As reported in \url{https://github.com/clarkkev/deep-coref}}) and $+0.52$ (15 hours) 
on the CoNLL metric for \newcite{clark-manning:2016:EMNLP2016} 
and ours, respectively. 
It is worth noting that our absolute scores are weaker than these of 
\newcite{clark-manning:2016:EMNLP2016}, as they build on top of a  similar but stronger 
mention-ranking baseline, which employs 
deeper neural networks and requires a much larger number of epochs to train 
(300 epochs, including pretraining). For the purpose of illustrating 
the proposed losses, we started with a simpler model by \newcite{P15-1137}
which requires a much smaller number of epochs, thus faster, to train (20 epochs, including 
pretraining).

\begin{table*}[h]
    \centering
    \begin{tabular}{l||c|c|c||c|c|c}
         & \multicolumn{3}{c}{Non-Anaphoric (FA)} & \multicolumn{3}{c}{Anaphoric (FN + WL)} \\
         & Proper & Nominal & Pronom. & Proper & Nominal & Pronom. \\
         \hline
        baseline & 630 & 714 & 1051 & 374 + 190 & 821 + 238 & 347 + 779 \\
        $L_{ec}$ & 529 & 609 & 904 & 438 + 182 & 924 + 220 & 476 + 740 \\
        \hdashline
        $L_{\beta=1,B^3}$ & 545 & 559 & 883 & 433 + 172 & 951 + 192 & 457 + 761 \\
        $L_{\beta=\sqrt{1.4},B^3}$ & 557 & 564 & 926 & 426 + 178 & 941 + 194 & 431 + 766 \\
        \hdashline
        $L_{\beta=1,LEA}$ & 513 & 547 & 843 & 456 + 170 & 960 + 191 & 513 + 740 \\
        $L_{\beta=\sqrt{1.8},LEA}$ & 577 & 591 & 1001 & 416 + 176 & 919 + 198 & 358 + 790 \\
    \end{tabular}
    \caption{Number of: 
    ``false anaphor'' (FA, a non-anaphoric mention marked as anaphoric), 
    ``false new'' (FN, an anaphoric mention marked as non-anaphoric), 
    and ``wrong link'' (WL, an anaphoric mention is linked to a wrong antecedent) errors on the development set.}   
    \label{table errors}
\end{table*}

\begin{figure}[t!]
\begin{mdframed}
    (a) [...] that {\color{blue} $_{13}$[the virus]} could mutate [...] /. In fact some health experts say $_{17}$[it]$^{13*}_{17,17}$ 's just a matter of time [...]
    
    (b) Walk a mile in {\color{violet} $_{157}$[our]} shoes that 's all I have to say because anybody who works in a nursing home will very quickly learn that these are very fragile patients /. {\color{violet} $_{165}$[We]$^{157}_{165*,157}$} did the very best {\color{violet} $_{167}$[we]$^{165}_{165,165}$} could in these situations [...]
    
\end{mdframed}
    
    \caption{Example predictions: the subscript before a mention
    is its index. The superscript / subscript after a mention indicates the 
    antecedent predicted by the baseline / $L_{\beta=1,B^3},L_{\beta=\sqrt{1.4},B^3}$. 
    Mentions with the same color are true coreferents. 
    ``*''s mark \textit{incorrect} decisions.}
    \label{fig examples}    
\end{figure}

\begin{figure*}[ht]
    \centering
    \includegraphics[width=0.47\textwidth]{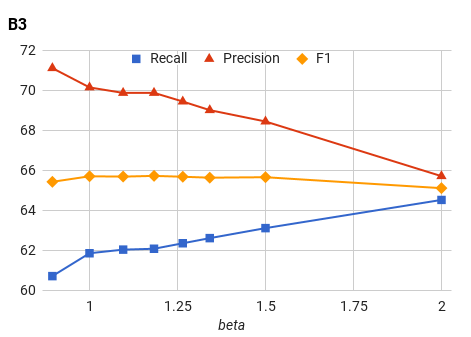}
    \includegraphics[width=0.47\textwidth]{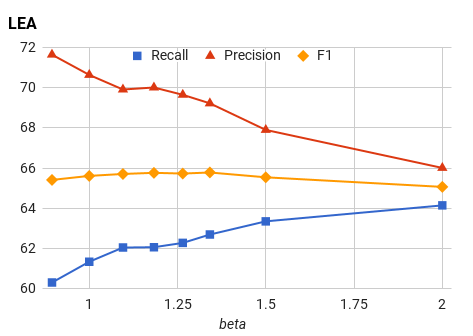}
    \caption{Recall, precision, F$_1$ (average of MUC, B$^3$, CEAF$_e$), 
    on the development set when training with $L_{\beta,B^3}$ (left) and $L_{\beta,LEA}$ (right). 
    Higher values of $\beta$ yield lower precisions but higher recalls.}
    \label{fig beta}
    \vspace{-1ex}
\end{figure*}

\subsection{Analysis}

Table~\ref{table errors} shows the breakdown of errors made by the baseline and 
our resolvers on the development set. 
The proposed resolvers make fewer 
``false anaphor'' and ``wrong link'' errors but more 
``false new'' errors compared to the baseline. This suggests that loss optimization prevents over-clustering, driving the precision up:
when antecedents are difficult to detect,  the self-link (i.e., $a_i = i$) is chosen.
When $\beta$ increases, they make more ``false anaphor'' and 
``wrong link'' errors but less ``false new'' errors. 

In Figure~\ref{fig examples}(a) the baseline, but not $L_{\beta=1,B^3}$ nor 
$L_{\beta=\sqrt{1.4},B^3}$, mistakenly links $_{17}$[it] with {\color{blue} $_{13}$[the virus]}. 
Under-clustering, on the other hand, is a problem for our resolvers with $\beta=1$:
in example (b), $L_{\beta=1,B^3}$ missed {\color{violet} $_{165}$[We]}.
This behaviour results in a reduced recall but the recall is not damaged severely, as we still obtain a better $F_1$ score.
We conjecture that this behaviour is a consequence 
of using the $F_1$ score in the objective, and, if undesirable, F$_\beta$ with $\beta > 1$ 
can be used instead. For instance, also in Figure~\ref{fig examples}, $L_{\beta=\sqrt{1.4}, B^3}$
correctly detects $_{17}$[it] as non-anaphoric and links {\color{violet} $_{165}$[We]} with 
{\color{violet} $_{157}$[our]}. 

Figure~\ref{fig beta} shows recall, precision, F$_1$ (average of MUC, B$^3$, CEAF$_e$), 
on the development set when training with $L_{\beta,B^3}$ and $L_{\beta,LEA}$. 
As expected, higher values of $\beta$ yield lower precisions but higher recalls. 
In contrast, F$_1$ increases until reaching the highest point when $\beta=\sqrt{1.4} \approx 1.18$ for 
$L_{\beta,B^3}$ ($\beta=\sqrt{1.8} \approx 1.34$ for $L_{\beta,LEA}$), it then decreases gradually.

\subsection{Discussion}

Because the resolvers are evaluated on F$_1$ score metrics, 
it should be that $L_{\beta,B^3}$ and $L_{\beta,LEA}$ perform 
the best with $\beta=1$. Figure~\ref{fig beta} and Table~\ref{table test} however 
do not confirm that: $\beta$ should be set with values a little bit larger than 1. 
There are two hypotheses. First, the statistical difference between the training set 
and the development set leads to the case that the optimal $\beta$ on one set can 
be sub-optimal on the other set. 
Second, in our experiments we fix $T = 1$, 
meaning that the relaxations might not be close to the true evaluation metrics enough. 
Our future work, to confirm/reject this, is to use annealing, i.e., gradually decreasing $T$ down to 
(but larger than) 0. 

Table~\ref{table test} shows that the difference between 
$L_{\beta,B^3}$ and $L_{\beta,LEA}$ in terms of accuracy 
is not substantial (although the latter is slightly better than the former). However,
one should expect that $L_{\beta,B^3}$ would outperform $L_{\beta,LEA}$ on B$^3$ metric
while it would be the other way around on LEA metric. 
It turns out that, B$^3$ and LEA behave quite similarly in non-extreme cases. 
We can see that in Figure 2, 4, 5, 6, 7 in \newcite{moosavi-strube:2016:P16-1}.

\section{Related work} 

Mention ranking and entity centricity are two main streams in 
the coreference resolution literature. 
Mention ranking \cite{denis2007ranking, D13-1203, TACL604, wiseman2015learning} 
considers local and independent decisions when choosing a correct antecedent 
for a mention. This approach is computationally efficient and currently 
dominant with state-of-the-art performance \cite{N16-1114, clark-manning:2016:EMNLP2016}. 
\newcite{P15-1137} propose to use simple neural networks to 
compute mention ranking scores and to use a heuristic loss to train the model.
\newcite{N16-1114} extend this by employing LSTMs to compute mention-chain 
representations which are then used to compute ranking scores. They call these 
representations \textit{global features}. 
\newcite{clark-manning:2016:EMNLP2016} build a similar resolver as in \newcite{P15-1137}
but much stronger 
thanks to deeper neural networks and ``better mention detection, 
more effective, hyperparameters, and more epochs of training''. 
Furthermore, using reward rescaling
they achieve the best performance in the literature on the English and Chinese portions 
of the CoNLL 2012 dataset. Our work is built upon mention ranking by turning a mention-ranking model 
into an entity-centric one. It is worth noting that although we use the model 
proposed by \newcite{P15-1137}, any mention-ranking models can be employed. 

Entity centricity 
\cite{wellner2003towards,poon2008joint,haghighi2010coreference, ma-EtAl:2014:EMNLP2014, clark-manning:2016:P16-1}, 
on the other hand, incorporates entity-level information to solve the problem. 
The approach can be top-down as in \newcite{haghighi2010coreference} where they  
propose a generative model. It can also be bottom-up by merging smaller clusters into 
bigger ones as in \newcite{clark-manning:2016:P16-1}. The method proposed by \newcite{ma-EtAl:2014:EMNLP2014}
greedily and incrementally adds mentions to previously built clusters using 
a prune-and-score technique. 
Importantly, employing imitation learning these two methods can optimize the resolvers
directly on evaluation metrics. 
Our work is similar to \newcite{ma-EtAl:2014:EMNLP2014} in the sense that 
our resolvers incrementally add mentions to previously built clusters. 
However, different from both 
\newcite{ma-EtAl:2014:EMNLP2014,clark-manning:2016:P16-1}, our resolvers do not use
any discrete decisions (e.g., merge operations). Instead, they seamlessly compute 
the probability that a mention refers to an entity from mention-ranking probabilities, 
and are optimized on differentiable relaxations of evaluation metrics.

Using differentiable relaxations of evaluation metrics as in our work is related to 
a line of research in reinforcement learning where a non-differentiable action-value function 
is replaced by a differentiable critic \cite{sutton1999policy, DBLP:conf/icml/SilverLHDWR14}. 
The critic is trained so that it is as close to the true action-value function as possible. 
This technique is applied to machine translation \cite{gu2017trainable} where 
evaluation metrics (e.g., BLUE) are non-differentiable. A disadvantage of using critics is that 
there is no guarantee that the critic converges to the true evaluation metric given 
finite training data. In contrast, our differentiable relaxations do not need 
to train, and the convergence is guaranteed as $T \rightarrow 0$. 

\section{Conclusions}
We have proposed 
\begin{itemize}
    \item a method for turning any mention-ranking resolver into an entity-centric one 
    by using a recursive formula to combine scores of individual local decisions, and
    \item differentiable relaxations for two coreference evaluation metrics, B$^3$ and LEA. 
\end{itemize}
Experimental results show that our approach outperforms the resolver 
by \newcite{N16-1114}, and gains a higher improvement over the baseline than that of
\newcite{clark-manning:2016:EMNLP2016} but with much shorter training time.

\section*{Acknowledgments}
We would like to thank Raquel Fernández, Wilker Aziz, Nafise Sadat Moosavi, 
and anonymous reviewers for their suggestions and
comments. The project was supported by the
European Research Council (ERC StG BroadSem
678254), the Dutch National Science Foundation
(NWO VIDI 639.022.518) and an Amazon Web Services
(AWS) grant.

\bibliography{ref}

\bibliographystyle{acl_natbib}

\end{document}